\documentclass[10pt,twocolumn]{article}

\usepackage[margin=0.72in,columnsep=0.24in]{geometry}
\usepackage{microtype}
\usepackage{graphicx}
\usepackage{subcaption}
\usepackage{booktabs}
\usepackage{tabularx}
\usepackage{amsmath}
\usepackage{amssymb}
\usepackage{mathtools}
\usepackage{amsthm}
\usepackage[round,authoryear]{natbib}
\usepackage{xcolor}
\usepackage[colorlinks=true,allcolors=blue!55!black]{hyperref}
\usepackage[capitalize,noabbrev]{cleveref}

\theoremstyle{plain}
\newtheorem{theorem}{Theorem}[section]
\newtheorem{proposition}[theorem]{Proposition}

\theoremstyle{definition}

\theoremstyle{remark}

\title{\textbf{A Stackelberg Framework for Resource-Aware LLM Agents:\\
Learning, Repair, and Conditional Guarantees}}
\author{Baoxun Wang\\
Platform and Content Group, Tencent}
\date{Technical Report --- Preprint --- \today}

\begin{document}

\maketitle

\begin{abstract}
Large language model (LLM) agents increasingly operate as multi-turn systems that must allocate context, prompt verbosity, and tool access under finite computational budgets.  Static thresholds are simple, but they are brittle under heterogeneous tasks and evolving session states.  We formulate resource governance as a contextual Stackelberg game: a controller commits to a quality target and a cost incentive, while an executor responds with resource actions over context, prompting, and tool usage.  We learn a conditional response model, optimize a leader policy against that model, and repair the resulting policy using real-API calibration and projection onto an empirically selected action set.  For the restricted game, we establish conditional guarantees for equilibrium existence, follower-response stability, safe-set projection, and transfer from a surrogate environment to the real environment under bounded value error.  The primary real-API experiment comprises 300 evaluated turns.  Relative to a conservative baseline, the selected repaired controller reduces mean token cost by 17.4\% (Welch $p=0.022$), while the measured quality difference is not statistically significant ($p=0.44$).  The theoretical results are conditional and the experiments do not estimate their regret or transfer constants; consequently, the evidence establishes a promising repaired operating point, not a certified real-system equilibrium.
\end{abstract}

\section{Introduction}

LLM agents are no longer isolated prompt-response systems.  Recent agentic systems combine reasoning, acting, tool use, reflection, and long-horizon interaction \citep{brown2020language,yao2023react,schick2023toolformer,shinn2023reflexion,wang2023voyager}.  In realistic multi-turn environments, an agent repeatedly decides how much conversation history to retain, whether to summarize or compress prior context, how detailed the next prompt should be, whether tools should be invoked, and how aggressively the current step should spend a limited budget.  These decisions jointly determine cost, latency, robustness, and answer quality.

Many agent systems implement such decisions with static control rules: fixed token-reserve thresholds, fixed compaction policies, fixed prompt templates, and fixed tool-access configurations.  Static rules are easy to implement, but they are brittle under heterogeneous workloads, especially when context placement and API cost-quality trade-offs matter \citep{liu2024lost,chen2023frugalgpt}.  A short factual query and a multi-file coding task should not receive the same context policy, prompt verbosity, or tool budget.  Conversely, blindly enabling all resources for every task can waste substantial computation.

This report studies the following question: can resource governance in LLM agents be learned as a leader-follower mechanism rather than hand-written as static execution rules?  We argue that this problem is naturally captured by a Stackelberg formulation, where a leader commits first and a follower subsequently responds \citep{stackelberg1934marktform,fudenberg1991game,conitzer2006computing}.  A global controller, such as a gateway or policy layer, observes the session state and commits to a quality target and a budget incentive.  A task executor observes these signals and responds with concrete resource actions.

We make four contributions.  First, we formulate multi-turn LLM-agent resource governance as a contextual leader-follower game with explicit commitment, budget state, and temporal consistency.  Second, we give conditional results for restricted-equilibrium existence, follower-response stability, repair feasibility, and surrogate-to-real regret transfer.  Third, we present a learning-and-repair pipeline over context management, prompt verbosity, and tool budget.  Fourth, we provide a structured evaluation hierarchy: a primary real-API cost-quality experiment, followed by calibration diagnostics, offline safety analysis, and non-interfering shadow validation.

\section{Theoretical Foundation}

\subsection{Stage-wise Contextual Stackelberg Game}

We use a stage-wise contextual Stackelberg abstraction for each agent turn, and later lift it to a finite-horizon policy objective \citep{stackelberg1934marktform,fudenberg1991game}.  Let $s_t \in \mathcal{S}$ denote the raw observable session features at turn $t$.  To make budget and temporal-commitment terms explicit, define the augmented state
\begin{equation}
    x_t := (s_t, B_t, \ell_{t-1}) \in \mathcal{X},
\end{equation}
where $B_t$ is the remaining budget state and $\ell_{t-1}$ is the previous leader signal.  The leader first proposes a raw signal $\tilde{\ell}_t=(\tilde{q}_t,\tilde{\alpha}_t)\in\widetilde{\mathcal{L}}$, which is projected to the feasible commitment set
\begin{equation}
    \ell_t=(q_t,\alpha_t)
    = \Pi_{\mathcal{L}}(\tilde{\ell}_t),
    \qquad
    \mathcal{L}:=[q_{\min},q_{\max}]\times[0,1].
    \label{eq:leader-projection}
\end{equation}
Here $q_t$ is a quality or completeness target and $\alpha_t$ is a cost-subsidy coefficient.  After observing $(x_t,\ell_t)$, the follower selects a resource action
\begin{equation}
    a_t = (c_t, p_t, u_t) \in \mathcal{A}:=[0,1]^3,
\end{equation}
where $c_t$ controls context retention, $p_t$ controls prompt verbosity, and $u_t$ controls tool budget.  The environment then produces token cost $T_t$, quality signal $Q_t$, and next augmented state $x_{t+1}$.

For a fixed augmented state $x$ and leader signal $\ell$, the follower's stage-wise best-response correspondence is
\begin{equation}
    \mathrm{BR}^{\mathrm{stg}}_F(x,\ell)
    := \arg\max_{a \in \mathcal{A}} \, U_F(x,\ell,a).
    \label{eq:follower-br}
\end{equation}
The raw output is projected before entering the game.  The leader's
stage-wise Stackelberg choice is therefore expressed over feasible signals as
\begin{equation}
\begin{aligned}
    \ell^*(x)
    &\in \arg\max_{\ell \in \mathcal{L}}
    U_L\bigl(x,\ell,a^*(x,\ell)\bigr), \\
    a^*(x,\ell)
    &\in \arg\max_{a\in\mathrm{BR}^{\mathrm{stg}}_F(x,\ell)}
    U_L(x,\ell,a).
\end{aligned}
    \label{eq:leader-br}
\end{equation}
If the best-response set is multi-valued, we either assume the learned follower response is deterministic or use the optimistic tie-breaking convention.  This convention is sufficient for our empirical deployment setting, where a trained follower model returns a concrete response action.

This is the essential distinction from direct resource control: the leader does not set $(c,p,u)$ directly; it sets $(q,\alpha)$ so as to induce a desirable follower response.

\subsection{Why This Matches Agent Resource Governance}

Many agent architectures separate a control plane from task execution, analogous to leader-follower resource-allocation settings studied in Stackelberg security games \citep{kiekintveld2009computing,tambe2011security}.  A policy layer manages memory, budget, and safety, while an executor completes the current task.  The controller's objective is long-horizon cost-quality stability, whereas the executor's local objective is task completion under the perceived cost of resources.  This partial misalignment motivates a commitment model rather than a purely cooperative hierarchy.

The subsidy $\alpha_t$ is the key mechanism.  It transforms physical token cost into the follower's perceived cost,
\begin{equation}
    C_F(x_t,a_t;\alpha_t) = (1-\alpha_t)\,T(x_t,a_t),
    \label{eq:effective-cost}
\end{equation}
so higher $\alpha_t$ makes costly but potentially higher-quality actions more attractive to the follower.  The target $q_t$ complements $\alpha_t$ by preventing purely cost-minimizing behavior: the follower is penalized for failing to satisfy the leader's quality commitment.

\subsection{Notation}

\begin{table}[t]
\caption{Core notation for the contextual Stackelberg formulation.}
\label{tab:notation}
\vskip 0.05in
\begin{center}
\begin{small}
\begin{tabularx}{\columnwidth}{@{}lX@{}}
\toprule
Symbol & Meaning \\
\midrule
$s_t$ & raw observable session features at turn $t$ \\
$x_t=(s_t,B_t,\ell_{t-1})$ & augmented state used by the controller \\
$B_t$ & remaining budget state \\
$\tilde{\ell}_t=(\tilde{q}_t,\tilde{\alpha}_t)$ & raw leader output before projection \\
$\widetilde{\mathcal{L}}$ & raw leader-output domain \\
$\ell_t=(q_t,\alpha_t)$ & projected quality target and subsidy \\
$a_t=(c_t,p_t,u_t)$ & follower context, prompt, and tool action \\
$T(x,a)$ & token cost under state-action pair \\
$Q(x,a)$ & quality or completeness signal \\
$T_0(x)$ & reference baseline token cost \\
$U_L,U_F$ & leader and follower utilities \\
$\pi_L,\pi_F$ & leader and follower policies \\
$\Pi_L,\Pi_F$ & restricted leader and follower policy classes \\
$P,\widehat{P}$ & real and estimated transition laws \\
$V_L^M$ & reduced leader value in environment $M$ \\
$\mathcal R_L^M,\mathcal R_F^M$ & restricted leader and follower regrets \\
\bottomrule
\end{tabularx}
\end{small}
\end{center}
\end{table}

In our implementation, $s_t$ contains task complexity, context length, turn index, average cost per turn, recent quality, cost trend, model condition, and coarse task type, while $B_t$ and $\ell_{t-1}$ are included in the augmented state.  A task-aware variant expands task type into a one-hot representation over six broad task categories.

\subsection{Modeling Scope and Assumptions}

The stage-wise Stackelberg formulation is an abstraction of one resource-allocation decision, not a proof that the executor is a perfectly rational economic agent.  We assume that (i) the controller's commitment signal is observable to the response model, (ii) the response distribution is stationary over the data-collection and optimization interval, and (iii) the projected action is implemented consistently by the runtime adapter.  These assumptions can fail under model updates, prompt-template changes, or tool-interface changes.

The environment transition is written as
\begin{equation}
    (T_t,Q_t,x_{t+1}) \sim P(\cdot\mid x_t,\ell_t,a_t),
    \label{eq:transition}
\end{equation}
where $P$ is unknown and is approximated by simulation, supervised surrogates, or direct API evaluation depending on the experiment.  This distinction is important: optimization under an estimated transition or payoff model yields a candidate operating point for that estimated environment, not a certified equilibrium for the real system.

\section{Payoff Design}

\subsection{Follower Utility}

The follower utility combines quality, effective cost, and target-gap penalty.  We write
\begin{equation}
\begin{aligned}
    U_F(x_t,\ell_t,a_t)
    &= w_Q Q(x_t,a_t)
       - w_C (1-\alpha_t) T(x_t,a_t) \\
    &\quad - w_G [q_t - Q(x_t,a_t)]_+,
\end{aligned}
\label{eq:follower-utility}
\end{equation}
where $[z]_+ := \max(z,0)$ and all weights are non-negative.  The term $[q_t-Q]_+$ ensures that the quality target is not merely advisory: if the leader commits to a high target, low-quality low-cost actions lose utility.  Imitation rewards used during response learning are treated as training surrogates, not as semantic payoffs in the game.  This separation keeps the equilibrium interpretation tied to resource governance rather than to a discriminator artifact.

\subsection{Leader Utility}

The leader is rewarded for saving cost relative to a reference policy while maintaining quality and avoiding degenerate target choices.  Let
\begin{equation}
    S(x_t,a_t) = \operatorname{clip}\left(1 - \frac{T(x_t,a_t)}{T_0(x_t)},\, -\kappa,\, 1\right)
    \label{eq:saving-ratio}
\end{equation}
be the clipped saving ratio relative to baseline cost $T_0(x_t)$.  A representative leader utility is
\begin{equation}
\begin{aligned}
    U_L(x_t,\ell_t,a_t)
    &= \omega_S S(x_t,a_t)
       - \omega_D [\tau_Q - Q(x_t,a_t)]_+ \\
    &\quad - \omega_B G_B(x_t,a_t) \\
    &\quad - \omega_\Delta \|\ell_t - \ell_{t-1}\|_2^2 .
\end{aligned}
\label{eq:leader-utility}
\end{equation}
where
\begin{equation}
    G_B(x_t,a_t)
    :=[\tau_B-B_t]_+[T(x_t,a_t)-\tau_T]_+
    \label{eq:budget-penalty}
\end{equation}
is a continuous soft penalty for expensive actions under low remaining budget.
This form avoids the discontinuity of a hard threshold indicator and makes the
regularity assumptions used below explicit.  A runtime may still use a hard
alert for monitoring, but that alert is not part of the theoretical payoff.
Because \cref{eq:leader-projection} already enforces
$q_t\in[q_{\min},q_{\max}]$, the game payoff does not require a second penalty
on infeasible raw targets.  The smoothness term discourages unstable incentive
changes across turns.

\subsection{Commitment and Temporal Consistency}

A practical Stackelberg controller should not change incentives arbitrarily at every token or micro-step.  We therefore interpret $(q_t,\alpha_t)$ as a short-horizon commitment, optionally held fixed for $K$ follower steps.  For the subsidy, the raw output $\tilde{\alpha}_t$ can be smoothed and projected by
\begin{equation}
    \alpha_t
    = \Pi_{[0,1]\cap[\alpha_{t-1}-\delta_\alpha,\,\alpha_{t-1}+\delta_\alpha]}
    \left(\rho\alpha_{t-1} + (1-\rho)\tilde{\alpha}_t\right).
    \label{eq:alpha-smoothing}
\end{equation}
This makes the leader signal temporally consistent, explicitly enforces the maximum per-turn subsidy change, and makes the signal easier for the follower policy to interpret.

\section{Learning Approximate Stackelberg Operating Points}

\subsection{Empirical Backward Induction}

Exact backward induction would require computing \cref{eq:follower-br} for every leader signal and then solving \cref{eq:leader-br}.  In LLM-agent systems, both the response function and payoff function are unknown, noisy, and expensive to evaluate.  We therefore use an empirical policy-level approximation.  Let $\Pi_L$ and $\Pi_F$ be the restricted policy classes used in training, and let $\widehat{P}$ denote the simulator or calibrated surrogate used for optimization.  Let $\mathfrak{BR}^{\widehat{P}}_F(\pi_L)$ denote a follower response within $\Pi_F$ under $\widehat{P}$.  We approximate
\begin{equation}
    \widehat{\pi}_F \approx \mathfrak{BR}^{\widehat{P}}_F(\pi_L),
    \qquad
    \widehat{\pi}_L \approx
    \arg\max_{\pi_L\in\Pi_L}
    \widehat{J}_L(\pi_L;\widehat{\pi}_F,\widehat{P}).
    \label{eq:empirical-bi}
\end{equation}
This is best understood as learning a candidate operating point within restricted function classes and an estimated environment rather than solving the unrestricted real game.

\subsection{Follower Response Learning}

We learn a conditional follower policy
\begin{equation}
    \pi_F(a_t \mid x_t,q_t,\alpha_t),
    \label{eq:follower-policy}
\end{equation}
which maps the leader's incentive signal and the augmented session state to resource actions.  With a GAIL-style discriminator $D_\psi$, the follower can be trained by the saddle-point objective \citep{ho2016gail,finn2016guided}
\begin{equation}
\begin{aligned}
    \min_{\pi_F}\max_{D_\psi}\;&
    \mathcal{J}_{\mathrm{GAIL}}(\pi_F,D_\psi) \\
    &= \mathbb{E}_{\mathcal{D}_E}
       \!\left[\log D_\psi(x,\ell,a)\right] \\
    &\quad + \mathbb{E}_{\pi_F}
       \!\left[\log(1-D_\psi(x,\ell,a))\right] \\
    &\quad - \lambda_H \mathcal{H}(\pi_F),
\end{aligned}
\label{eq:gail-objective}
\end{equation}
where $\mathcal{D}_E$ denotes expert or heuristic response data and $\mathcal{H}$ is an entropy regularizer.  The important modeling point is that $\pi_F$ and $D_\psi$ are conditioned on $(q,\alpha)$; the model learns a response surface, not a single average execution rule.

\subsection{Leader Optimization}

Given $\widehat{\pi}_F$, the leader maximizes the expected discounted utility under the optimization environment $\widehat{P}$,
\begin{equation}
    \widehat{J}_L(\pi_L;\widehat{\pi}_F,\widehat{P})
    = \mathbb{E}_{\tau\sim(\pi_L,\widehat{\pi}_F,\widehat{P})}
    \left[\sum_{t=0}^{H-1}\gamma^t U_L(x_t,\ell_t,a_t)\right],
    \label{eq:leader-objective}
\end{equation}
where $\tilde{\ell}_t\sim\pi_L(\cdot\mid x_t)$, $\ell_t=\Pi_{\mathcal{L}}(\tilde{\ell}_t)$, and $a_t\sim\widehat{\pi}_F(\cdot\mid x_t,q_t,\alpha_t)$.  PPO/HRL-style updates are used to optimize this objective in practice \citep{sutton1999between,schulman2017ppo}.

\subsection{Approximate Equilibrium Interpretation}

For an environment $M\in\{P,\widehat P\}$ and $i\in\{L,F\}$, define
\begin{equation}
 J_i^M(\pi_L,\pi_F)
 :=
 \mathbb E_{\tau\sim(\pi_L,\pi_F,M)}
 \left[\sum_{t=0}^{H-1}\gamma^t U_i(x_t,\ell_t,a_t)\right].
 \label{eq:policy-value}
\end{equation}
We assume below that the restricted classes are chosen so that the relevant
maxima are attained; finite policy classes are sufficient.  Define the
restricted follower best-response correspondence
\begin{equation}
    \mathfrak{BR}^{M}_F(\pi_L)
    :=\arg\max_{\pi_F\in\Pi_F}J_F^M(\pi_L,\pi_F).
\end{equation}
Let $b_M(\pi_L)\in\mathfrak{BR}^{M}_F(\pi_L)$ be an optimistic selector, chosen to maximize leader utility among follower best responses, and define the reduced Stackelberg value
\begin{equation}
    V_L^M(\pi_L)
    :=J_L^M\bigl(\pi_L,b_M(\pi_L)\bigr).
    \label{eq:reduced-leader-value}
\end{equation}
Unlike a deviation that artificially freezes the follower, $V_L^M$ allows the follower to respond again after every leader deviation.

For a candidate pair $(\pi_L,\pi_F)$, define restricted follower and leader regrets by
\begin{align}
    \mathcal R_F^M(\pi_L,\pi_F)
    &:=
    \sup_{\pi'_F\in\Pi_F}J_F^M(\pi_L,\pi'_F)
    -J_F^M(\pi_L,\pi_F), \label{eq:follower-regret}\\
    \mathcal R_L^M(\pi_L)
    &:=
    \sup_{\pi'_L\in\Pi_L}V_L^M(\pi'_L)
    -V_L^M(\pi_L). \label{eq:leader-regret}
\end{align}
The pair is an $(\epsilon_L,\epsilon_F)$ restricted approximate Stackelberg operating point in $M$ when
\begin{equation}
    \mathcal R_L^M(\pi_L)\leq\epsilon_L,
    \qquad
    \mathcal R_F^M(\pi_L,\pi_F)\leq\epsilon_F.
    \label{eq:approx-operating-point}
\end{equation}
The leader regret is always non-negative because it compares reduced
Stackelberg values, while the follower regret measures whether the candidate
response is close to a restricted best response.  This definition does not
assert that GAIL produces a best response.  Any discrepancy between the
learned follower and the restricted response oracle appears explicitly as
follower regret $\epsilon_F$.  The actual leader payoff under an approximate
follower can differ from $V_L^M(\pi_L)$; controlling that implementation gap
requires an additional response-approximation argument and is not claimed here.

\section{Conditional Theoretical Guarantees}

The results below establish properties of the restricted mathematical model.  They are conditional statements, not empirical certificates for the learned neural policies.

\subsection{Existence of a Stage-Wise Equilibrium}

\begin{theorem}[Restricted equilibrium existence]
\label{thm:existence}
Fix a state $x$.  Suppose $\mathcal L$ and $\mathcal A$ are nonempty compact sets and $U_L(x,\ell,a)$ and $U_F(x,\ell,a)$ are continuous in $(\ell,a)$.  Under optimistic tie breaking, the stage-wise game admits a Stackelberg solution $(\ell^\star,a^\star)$.
\end{theorem}

\begin{proof}
Continuity of $U_F$ and compactness of $\mathcal A$ imply that
$\mathrm{BR}^{\mathrm{stg}}_F(x,\ell)$ is nonempty and compact for every
$\ell$.  Its graph is compact by the maximum theorem.  Therefore
\[
    \mathcal G_x
    :=\{(\ell,a):\ell\in\mathcal L,\,
    a\in\mathrm{BR}^{\mathrm{stg}}_F(x,\ell)\}
\]
is nonempty and compact.  The continuous function $U_L$ attains a maximum
on $\mathcal G_x$.  Any maximizer $(\ell^\star,a^\star)$ is an optimistic
Stackelberg solution.
\end{proof}

The box-valued action sets in \cref{eq:leader-projection} satisfy compactness.
For the payoff specification above, continuity also holds when $T(x,\cdot)$
and $Q(x,\cdot)$ are continuous and $T_0(x)>0$.  The result guarantees
existence for the stated restricted game; it does not imply that the learning
algorithm finds the maximizer.

\subsection{Stability of the Follower Response}

\begin{proposition}[Lipschitz best response]
\label{prop:response-stability}
Let $\mathcal A$ be nonempty, compact, and convex.  Suppose
$a\mapsto U_F(x,\ell,a)$ is differentiable and $\mu$-strongly concave for
every $\ell$, and
\[
 \|\nabla_a U_F(x,\ell,a)-\nabla_a U_F(x,\ell',a)\|
 \leq L_{a\ell}\|\ell-\ell'\|.
\]
Then the follower best response is unique and satisfies
\begin{equation}
    \|a^\star(x,\ell)-a^\star(x,\ell')\|
    \leq \frac{L_{a\ell}}{\mu}\|\ell-\ell'\|.
    \label{eq:br-lipschitz}
\end{equation}
\end{proposition}

\begin{proof}
Strong concavity gives uniqueness.  Let $a=a^\star(x,\ell)$ and
$a'=a^\star(x,\ell')$.  The first-order variational inequalities for the two
constrained maximization problems are
\begin{align*}
 \langle\nabla_aU_F(x,\ell,a),a'-a\rangle&\leq0,\\
 \langle\nabla_aU_F(x,\ell',a'),a-a'\rangle&\leq0.
\end{align*}
Adding them and using strong concavity at fixed $\ell$ gives
\begin{align*}
 0
 &\leq
 \langle\nabla_aU_F(x,\ell,a)
 -\nabla_aU_F(x,\ell',a'),a-a'\rangle\\
 &\leq
 -\mu\|a-a'\|^2\\
 &\quad+
 \langle\nabla_aU_F(x,\ell,a')
 -\nabla_aU_F(x,\ell',a'),a-a'\rangle.
\end{align*}
Consequently,
\[
 \mu\|a-a'\|^2
 \leq
 \|\nabla_a U_F(x,\ell,a')-\nabla_a U_F(x,\ell',a')\|
 \|a-a'\|.
\]
Applying the assumed cross-Lipschitz bound and cancelling $\|a-a'\|$
proves \cref{eq:br-lipschitz}; the zero-distance case is immediate.
\end{proof}

Thus, under these regularity conditions, smoothing the commitment signal also
bounds changes in the induced resource action.  The exact hinge term in
\cref{eq:follower-utility} is non-smooth at zero; direct application of the
differentiable proposition therefore requires a smooth hinge approximation
and a strongly concave response parameterization.  Neither property is
verified for the empirical follower, so the proposition motivates rather than
certifies the deployed smoothing rule.

\subsection{Guarantees for Action Repair}

\begin{proposition}[Feasibility and projection loss]
\label{prop:repair}
Let $\mathcal C\subseteq\mathbb R^d$ be nonempty, closed, and convex, and let
$R(a)=\Pi_{\mathcal C}(a)$ be Euclidean projection.  Then
$R(a)\in\mathcal C$ for every $a$.  If $U_L(x,\ell,\cdot)$ is
$L_U$-Lipschitz, then
\begin{equation}
    U_L(x,\ell,R(a))
    \geq U_L(x,\ell,a)-L_U\|R(a)-a\|.
    \label{eq:repair-loss}
\end{equation}
If $\sup_a\|R(a)-a\|\leq d_{\max}$ on the raw action domain, the loss is at
most $L_Ud_{\max}$.
\end{proposition}

\begin{proof}
Existence, uniqueness, and feasibility of Euclidean projection follow from
closedness and convexity of $\mathcal C$.  The Lipschitz property gives
$|U_L(x,\ell,R(a))-U_L(x,\ell,a)|\leq L_U\|R(a)-a\|$, which yields the
stated lower bound.
\end{proof}

Taking $\mathcal C=\mathcal A_{\mathrm{real}}$ gives a formal feasibility
guarantee for the box projection used in deployment.  The loss bound applies
when the composed payoff is Lipschitz in $a$; on the compact action domain this
follows, for example, if $T$ and $Q$ are Lipschitz and $T_0$ is bounded away
from zero.  The result bounds only the distortion caused by projection; it
does not show that the empirically selected box is optimal or universally
safe.  It is a stage-wise payoff bound and does not by itself control
finite-horizon value loss, because a repaired action may also change future
states.

\subsection{Surrogate-to-Real Regret Transfer}

\begin{theorem}[Conditional transfer bound]
\label{thm:transfer}
Suppose that, over the restricted policy classes,
\begin{align}
 \sup_{\pi_L,\pi_F}
 |J_F^P(\pi_L,\pi_F)-J_F^{\widehat P}(\pi_L,\pi_F)|
 &\leq\delta_F, \label{eq:joint-value-error}\\
 \sup_{\pi_L}
 |V_L^P(\pi_L)-V_L^{\widehat P}(\pi_L)|
 &\leq\delta_L. \label{eq:reduced-value-error}
\end{align}
If $(\widehat\pi_L,\widehat\pi_F)$ is an
$(\epsilon_L,\epsilon_F)$ restricted approximate operating point under
$\widehat P$, then under $P$,
\begin{equation}
 \mathcal R_L^P(\widehat\pi_L)
 \leq\epsilon_L+2\delta_L,\qquad
 \mathcal R_F^P(\widehat\pi_L,\widehat\pi_F)
 \leq\epsilon_F+2\delta_F.
 \label{eq:transfer-bound}
\end{equation}
\end{theorem}

\begin{proof}
For the follower,
\begin{align*}
 \mathcal R_F^P(\widehat\pi_L,\widehat\pi_F)
 &=
 \sup_{\pi'_F}J_F^P(\widehat\pi_L,\pi'_F)
 -J_F^P(\widehat\pi_L,\widehat\pi_F)\\
 &\leq
 \sup_{\pi'_F}J_F^{\widehat P}(\widehat\pi_L,\pi'_F)
 -J_F^{\widehat P}(\widehat\pi_L,\widehat\pi_F)
 +2\delta_F\\
 &\leq\epsilon_F+2\delta_F.
\end{align*}
For the leader, \cref{eq:reduced-value-error} gives
\begin{align*}
 \mathcal R_L^P(\widehat\pi_L)
 &=
 \sup_{\pi'_L}V_L^P(\pi'_L)
 -V_L^P(\widehat\pi_L)\\
 &\leq
 \sup_{\pi'_L}V_L^{\widehat P}(\pi'_L)
 -V_L^{\widehat P}(\widehat\pi_L)
 +2\delta_L\\
 &\leq\epsilon_L+2\delta_L.
\end{align*}
\end{proof}

The theorem formalizes the role of calibration and repair: restricting
policies away from regions with severe simulator mismatch can reduce the
relevant uniform value errors.  The reduced-value assumption
\cref{eq:reduced-value-error} is separate from a uniform error bound on
$J_L$: because the follower best-response correspondence may change between
$P$ and $\widehat P$, one does not imply the other without additional response
stability or margin assumptions.  The present experiments do not estimate
$\epsilon_i$ or $\delta_i$, and low observed surrogate correlation prevents us
from assuming that these constants are small.

\section{Learning Setup}

\subsection{Expert-Style Demonstrations}

The follower response model is trained from expert-style demonstrations rather than from simulator rollouts alone.  The corpus combines public conversational data \citep{zhao2024wildchat} with privacy-filtered agent-task traces.  To avoid disclosing organization-specific datasets or collection procedures, we report the record schema and aggregate modeling role but omit source-specific counts, platform identifiers, and workflow metadata.  Each demonstration has the form
\begin{equation}
    (s_t, q_t, \alpha_t, T^{\mathrm{ratio}}_t, a_t),
\end{equation}
where $(q_t,\alpha_t)$ are the leader conditioning variables, $T^{\mathrm{ratio}}_t$ is the remaining-budget ratio, and $a_t=(c_t,p_t,u_t)$ is the resource action.  These labels are heuristic preferences, not human gold annotations and not observed best responses.  This distinction limits the game-theoretic interpretation of the learned follower.

\subsection{State and Action Schemas}

The scalar-state controller uses a 9-dimensional state with scalar task type and a 3-dimensional resource action.  The task-aware controller uses the same number of expert demonstrations but migrates the state to a 14-dimensional representation by replacing scalar task type with a six-way one-hot encoding over casual chat, simple question answering, text writing, code generation, data analysis, and complex reasoning.

\begin{table}[t]
\caption{Training setup for the two reported controller families.}
\label{tab:training-setup}
\vskip 0.05in
\begin{center}
\begin{small}
\begin{tabular}{@{}lcc@{}}
\toprule
Controller & State dimension & Action dimension \\
\midrule
Scalar-state & 9 & 3 \\
Task-aware & 14 & 3 \\
\bottomrule
\end{tabular}
\end{small}
\end{center}
\end{table}

The action dimensions are shared across versions: context strategy, prompt strategy, and tool budget.  The task-aware change is therefore not a larger imitation dataset, but a more explicit task representation and a different payoff-repair setting.

\subsection{Real-LLM Surrogate Data}

To improve payoff realism, we additionally collect a controlled real-LLM response grid for surrogate modeling.  This dataset is not used as direct GAIL imitation data.  Instead, it trains token and quality predictors that serve as an estimated payoff oracle.  The grid crosses coarse task types, two executor-model conditions, multiple context and prompt levels, bounded tool-budget levels, and repeated trials.  The surrogate features contain task indicators, model-condition indicators, and the three resource-action values.

\begin{table}[t]
\caption{Real-LLM surrogate dataset and held-out prediction quality.}
\label{tab:surrogate-data}
\vskip 0.05in
\begin{center}
\begin{small}
\begin{tabular}{@{}lcc@{}}
\toprule
Quantity & Value & Note \\
\midrule
Samples & 2400 & fixed train/test split \\
Token $R^2$ & 0.912 & MAE 65.5 tokens \\
Quality $R^2$ & 0.427 & MAE 0.050 \\
\bottomrule
\end{tabular}
\end{small}
\end{center}
\end{table}

The token surrogate is accurate enough to support cost-sensitive policy repair, while the quality surrogate remains only moderately predictive, consistent with known challenges in LLM-based evaluation \citep{zheng2023judging,liang2022holistic}.  This is why the final claims rely on real-API evaluation and passive shadow validation rather than simulator-only quality scores.

\subsection{Training Procedure}

Both reported controllers follow the same two-stage training pattern.  Phase 1 trains the conditional follower response with GAIL-style imitation.  Phase 2 freezes the follower and trains the leader with PPO-style updates over quality targets and subsidy coefficients.  Training rewards are optimization diagnostics only: because controller variants use different state encodings and payoff scaling, their raw rewards are not comparable and are omitted from the efficacy analysis.

\section{Payoff Misspecification and Real-API Repair}

\subsection{Wrong Payoff, Wrong Equilibrium}

A central lesson from our experiments is that payoff misspecification is not a minor evaluation issue, echoing sim-to-real concerns observed in other learned-control settings \citep{tobin2017domain}.  If simulated token costs, quality responses, or tool risks differ from real LLM API behavior, then a learned Stackelberg policy may converge to the wrong operating point.  A policy that is safe and efficient in simulation can become unsafe or dominated under real API dynamics.

\subsection{Real-API Calibration}

We therefore treat real-API calibration as part of the method.  Empirical probes revealed that token response is easier to calibrate than quality response.  Quality estimates remain noisy and judge-dependent, while token costs exhibit more consistent action elasticity.  This motivates a conservative repair strategy: constrain resource actions to empirically safe ranges rather than trusting raw policy outputs.

\subsection{Action-Space Projection}

The deployable repaired policy projects continuous actions into a real-safe feasible set.  Define
\begin{equation}
    \mathcal{A}_{\mathrm{real}}
    := [0.10,0.40]\times[0.35,0.65]\times[0.00,0.50].
\end{equation}
The deployed action is
\begin{equation}
    a_t^{\mathrm{real}} = \Pi_{\mathcal{A}_{\mathrm{real}}}(a_t),
    \qquad
    a_t^{\mathrm{real}}=(c_t^{\mathrm{real}},p_t^{\mathrm{real}},u_t^{\mathrm{real}}).
\end{equation}
This projection eliminates action regimes associated with prompt saturation, excessive context retention, or unrestricted tool budgets.  By \cref{prop:repair}, it guarantees membership in the selected box and admits a conditional utility-distortion bound.  It does not prove that the box itself is optimal or that all relevant real-world constraints have been encoded; learned constrained methods such as CPO provide a relevant alternative \citep{achiam2017cpo}.  We refer to the primary repaired policy as the scalar-state repaired controller.

\section{Coding-Aware Tool-Budget Repair}

Historical replay identified an additional risk: the repaired policy can be overly conservative for coding-like tasks, especially in tool budget.  This is undesirable because code editing, debugging, and repository inspection often require tools even when the global policy prefers low-cost actions.

We therefore add a coding-aware repair head for the shadow recommendation:
\begin{equation}
    u_t^{\mathrm{final}} = \min(0.50, \max(u_t^{\mathrm{base}}, \hat{u}_t^{\mathrm{need}})).
\end{equation}
The target is not historical tool imitation.  Instead, it estimates the minimum necessary tool budget for coding-like tasks while retaining a hard cap at 0.50.  In a historical replay study, the repair sharply reduces required-coding low-tool misses, keeps the tool trap rate at 0.0\%, and trades some estimated token savings for task safety.

\section{Passive Shadow Evaluation Protocol}

\subsection{Runtime Non-Interference}

Before active deployment, we evaluate the controller with a passive shadow protocol.  The purpose is to observe recommendation distributions under real workflows while preserving the behavior policy, a setting related to logged-policy evaluation \citep{bottou2013counterfactual,swaminathan2015counterfactual,agarwal2017effective}.  Let $\mu$ denote the actual execution policy and let $g_{\widehat{\pi}}$ denote the complete shadow recommendation map, including the learned leader, learned response model, projection, and any repair layer.  In shadow mode,
\begin{equation}
    a_t^{\mathrm{exec}} \sim \mu(\cdot\mid h_t),
    \qquad
    a_t^{\mathrm{shadow}} = g_{\widehat{\pi}}(x_t).
\end{equation}
The runtime contract excludes $a_t^{\mathrm{shadow}}$ from the inputs used to generate the user-visible action.  A causal statement of the intended contract is
\begin{equation}
\begin{aligned}
    &P\!\left(a_t^{\mathrm{exec}}\mid h_t,x_t,
      \operatorname{do}(a_t^{\mathrm{shadow}}=a)\right) \\
    &\qquad =
    P\!\left(a_t^{\mathrm{exec}}\mid h_t,x_t\right),
    \quad \forall a,
\end{aligned}
    \label{eq:shadow-noninterference}
\end{equation}
following the intervention notation of \citet{pearl2009causality}.  This is a software-architecture guarantee, not a statistical independence claim: $a_t^{\mathrm{exec}}$ and $a_t^{\mathrm{shadow}}$ may be correlated because both depend on the task and history.  The recommendation is logged but does not modify responses, context selection, tool calls, tool ordering, or task execution.

\subsection{Logged Counterfactual Record}

Each shadow event records a privacy-minimized counterfactual tuple
\begin{equation}
\begin{aligned}
    z_t = (&\tau_t,\, \hat{g}_t,\, a_t^{\mathrm{shadow}},\, f_t, \\
           &\hat{Q}_t,\, m_t).
\end{aligned}
\end{equation}
where $\tau_t$ is a coarse time or turn index, $\hat{g}_t$ is a coarse task category, $a_t^{\mathrm{shadow}}=(c_t,p_t,u_t)$ is the recommended resource action, $f_t$ is a fallback indicator, $\hat{Q}_t$ is an optional quality proxy, and $m_t$ contains privacy-filtered metadata.  This record is sufficient for estimating recommendation distributions, trap rates, fallback rates, and task coverage, while avoiding dependence on raw user prompts or assistant responses.

\subsection{Privacy-Preserving Aggregation}

For group analysis, we export only aggregate summaries.  These summaries intentionally exclude raw prompts, assistant messages, free-form notes, direct identifiers, and per-turn records.  Per-turn raw records are not used for cross-participant reporting.  Consequently, passive shadow evaluation can validate runtime stability and bounded action distributions, but it cannot by itself estimate realized token savings unless actual token usage is explicitly logged.

\section{Experimental Evaluation}

\subsection{Research Questions and Evaluation Design}

The empirical study has one primary question: \textbf{does the selected repaired controller reduce real-API token cost relative to a conservative baseline without detectable quality degradation?}  The 300-turn head-to-head evaluation is the sole confirmatory efficacy experiment.  All remaining analyses serve narrower roles: calibration probes diagnose payoff misspecification, offline replay identifies unsafe action patterns and distribution shift, and passive shadow logging evaluates runtime non-interference and observability.  These supporting analyses are not treated as independent demonstrations of online cost savings.

\subsection{Main Experiment: Real-API Cost-Quality Evaluation}

The main experiment uses 20 episodes, three turns per episode, and five strategies, for 300 evaluated turns.  It uses an API-served executor LLM and a separate lightweight LLM judge.  Routing and platform implementation details are omitted because they are not necessary to interpret the algorithmic comparison.  The strongest deployable result is the scalar-state repaired controller: compared with the conservative baseline, it reduces tokens per turn from 703.8 to 581.1, a 17.4\% reduction.  The token reduction is statistically significant with Welch $p=0.022$, while quality changes from 0.899 to 0.894 with no statistically significant difference ($p=0.44$).  This latter result is an absence of detected difference, not a formal non-inferiority conclusion, because a non-inferiority margin and power analysis were not pre-specified.  Quality-per-token efficiency improves from 1.278 to 1.694 when scaled by $10^3$, a 32.6\% gain.

\begin{table}[t]
\caption{Main real-API experiment over 300 evaluated turns.  The reported Stackelberg variants cover different Pareto roles.}
\label{tab:main-real-api}
\vskip 0.05in
\begin{center}
\begin{small}
\begin{tabular}{@{}lcccc@{}}
\toprule
Policy & Tok. & Q & Eff. & Role \\
\midrule
Middle & 933.7 & 0.888 & 0.951 & dominated \\
Conserv. & 703.8 & 0.899 & 1.278 & baseline \\
Task-aware raw & 263.5 & 0.824 & 3.127 & low-cost \\
Scalar+R & 581.1 & 0.894 & 1.694 & sweet spot \\
Task-aware+R & 817.2 & 0.904 & 1.106 & quality \\
\bottomrule
\end{tabular}
\end{small}
\end{center}
\end{table}

The Pareto frontier is informative rather than reducible to one winner.  The task-aware raw controller is an ultra-low-cost corner with lower quality, the scalar-state repaired controller is the efficiency sweet spot, and the task-aware repaired controller is the highest-quality operating point.  Three of the four non-dominated points are Stackelberg variants.

\begin{table}[t]
\caption{Task-level real-API comparison between the scalar-state repaired controller and the conservative baseline.  Each entry is tokens / quality.}
\label{tab:task-real-api}
\vskip 0.05in
\begin{center}
\begin{small}
\begin{tabular}{@{}lcc@{}}
\toprule
Task & Scalar+R & Conservative \\
\midrule
Casual chat & 572 / 0.89 & 697 / 0.90 \\
Code gen. & 830 / 0.88 & 975 / 0.88 \\
Complex reas. & 400 / 0.88 & 724 / 0.90 \\
Simple QA & 535 / 0.92 & 666 / 0.90 \\
Text writing & 604 / 0.89 & 627 / 0.90 \\
\bottomrule
\end{tabular}
\end{small}
\end{center}
\end{table}

\subsection{Supporting Analysis I: Why Repair Is Necessary}

The simulator and action probes are used to explain why repair is required; they are not treated as efficacy evidence.  In simulation, the learned policy can occupy a high-quality Pareto region while consuming more tokens than conservative fixed baselines.  This is consistent with the formulation---the controller seeks a cost-quality operating point, not unconditional token minimization---but it also shows that simulator-only success is sensitive to payoff scaling and baseline choice.

A 630-record sim-to-real calibration study confirms that the simulator is an imperfect payoff oracle.  Real token cost averages 347 tokens versus 827 simulated tokens, with Pearson correlation 0.355.  Real quality averages 0.887 versus 0.856 simulated quality, but the quality correlation is only 0.028.  A global token correction fit gives $T_{\mathrm{corr}}=0.3167T_{\mathrm{sim}}+85.4$, but the low quality correlation makes simulator-only equilibrium claims unreliable.

Real-API action probes further explain why payoff repair is necessary.  Context strategy has little single-turn elasticity, prompt verbosity can increase token cost by more than $20\times$, and high tool-budget hints can trigger phantom-tool failures that reduce both tokens and quality.  In a multi-turn context-growth probe, cumulative tokens by turn five increase from 3199 at low context to 9586 at full context, with nearly linear growth for full context ($R^2=0.997$).

\begin{table}[t]
\caption{Supporting calibration diagnostics.  These results motivate repair and are not confirmatory efficacy tests.}
\label{tab:action-probes}
\vskip 0.05in
\begin{center}
\begin{small}
\begin{tabularx}{\columnwidth}{@{}lXX@{}}
\toprule
Diagnostic & Observation & Interpretation \\
\midrule
Token calibration & real 347; sim 827; $r=.355$ & moderate mismatch \\
Quality calibration & real .887; sim .856; $r=.028$ & surrogate unreliable \\
Context probe & $+3.9\%$ tokens; $\Delta Q=.025$ & weak one-turn effect \\
Prompt probe & $+2162\%$ tokens; $\Delta Q=.135$ & strong nonlinearity \\
Tool probe & $-68\%$ tokens; $\Delta Q=-.250$ & failure mode observed \\
\bottomrule
\end{tabularx}
\end{small}
\end{center}
\end{table}

Empirically, raw policies often enter unsafe action regimes.  An early raw controller has a 99.3\% any-trap rate, while the scalar-state raw controller still has a 79.7\% any-trap rate.  After elasticity-based projection, the same policy family has zero observed trap-rate violations in the offline action audit.

\begin{table}[t]
\caption{Action repair audit.  Trap means any of $c>0.9$, $p>0.7$, or $u>0.5$.}
\label{tab:repair-audit}
\vskip 0.05in
\begin{center}
\begin{small}
\begin{tabular}{@{}lcc@{}}
\toprule
Mode & Mean action & Trap \\
\midrule
Early raw & $[0.83,0.82,0.84]$ & 99.3\% \\
Scalar raw & $[0.19,0.80,0.18]$ & 79.7\% \\
Early + proj. & $[0.35,0.59,0.42]$ & 0.0\% \\
Scalar + proj. & $[0.16,0.60,0.09]$ & 0.0\% \\
\bottomrule
\end{tabular}
\end{small}
\end{center}
\end{table}

\subsection{Supporting Analysis II: Offline Risk Discovery}

We replayed a privacy-filtered set of historical agent interactions.  This is an offline diagnostic, not an active A/B test: the policy recommends actions, while token usage is estimated through an elasticity model.  The replay's main contribution is risk discovery rather than a savings claim.  It reveals a distribution-shift failure mode in which coding-like tasks receive insufficient tool budget despite satisfying global action thresholds.

The replay motivates a coding-aware repair head that estimates minimum necessary tool budget rather than imitating historical full-tool behavior.  On the same replay set, the repair reduces the required-coding low-tool miss rate from 64.1\% to 1.0\%, preserves a zero observed tool-threshold violation rate, and reduces estimated savings from 47.2\% to 34.9\%.  These numbers quantify an offline safety-cost trade-off; they do not establish online task success or realized savings.

\subsection{Supporting Analysis III: Passive Shadow Runtime Check}

Passive shadow aggregates show bounded recommendations and no observed violations of the monitored action thresholds.  They also reveal environment-dependent fallback behavior.  The strongest supported conclusion is therefore limited: the recommendation and logging path can run without intentionally affecting execution, and the observed recommendations remain within configured bounds.  Because the shadow policy is not executed and actual token labels are unavailable, these logs cannot estimate realized savings under standard logged-feedback methods \citep{swaminathan2015counterfactual,agarwal2017effective}.

\section{Discussion}

The primary experiment supports a narrower claim than the full theoretical framing: under the evaluated tasks, models, and API configuration, one repaired controller reduces mean token use relative to a conservative baseline without a statistically detected quality difference.  The supporting analyses explain why real-API repair is needed and identify failure modes that are hidden by aggregate action thresholds.  They do not independently establish online effectiveness.

At the same time, the evidence remains preliminary.  The real-API evaluation is finite, historical replay relies on estimated token costs, and passive shadow aggregates lack actual token usage.  Moreover, coding-like online coverage is still limited, so the coding-aware repair is supported by offline replay rather than active A/B results.

\section{Runtime Integration}

The formulation applies to agent runtimes that separate policy control from task execution.  Before each model step, a policy layer can construct the state vector from coarse session metadata, call the resource controller, and map the returned action to context compaction, prompt style, and a bounded tool budget.  A conservative deployment path consists of passive shadowing, a small canary, a controlled A/B test, and only then a broader rollout with fallback.

The minimum observability fields are actual token usage, latency, fallback category, base and repaired action summaries, tool count, coarse task type, and a task-appropriate outcome measure.  Without these fields, shadow logs can validate bounded recommendations and basic runtime behavior but not realized cost savings or task success.

\section{Limitations and Unfinished Work}

This study is a research prototype with preliminary real-API evidence.  The following limitations define the boundary of the current claims.

\paragraph{Conditional theory without a numerical certificate.}
\Cref{thm:existence,prop:response-stability,prop:repair,thm:transfer} establish properties of the restricted model under compactness, continuity, strong-concavity, Lipschitz, and bounded-transfer-error assumptions.  The follower used in experiments is a learned conditional response model, not an exact best-response oracle, and projection plus task-specific repair changes the deployed policy after training.  We do not compute the restricted regrets $\epsilon_i$ or transfer errors $\delta_i$, nor verify strong concavity of the learned response surface.  The current result is therefore a repaired candidate operating point with conditional guarantees, not a numerically certified real-system equilibrium.

\paragraph{Limited statistical design.}
The primary experiment is finite and uses an LLM judge.  The reported $p=0.44$ does not prove quality equivalence or non-inferiority; no non-inferiority margin or prospective power analysis was pre-specified.  Confidence intervals, repeated evaluation across independent judges, and a pre-registered confirmatory study remain to be completed.

\paragraph{Quality measurement is task-dependent.}
The current quality channel relies on LLM judges and weak proxies, which are useful for early comparison but insufficient for complex tasks \citep{zheng2023judging,liang2022holistic}.  Coding tasks require test pass rates and patch correctness; data-analysis tasks require verifiable numerical outputs; writing and research tasks may require rubric-based or human evaluation.  Integrating these outcome measures is unfinished.

\paragraph{No active online validation.}
Passive shadowing checks the logging path and recommendation bounds, but it does not execute the policy.  Actual online token savings, latency effects, tool-use changes, and task-success impact remain unknown.  Logged-data estimators can support future sequential off-policy analysis \citep{jiang2016doubly}, but the present logs lack the action propensities and outcome coverage required for such estimates.  A canary and controlled A/B test with explicit stopping rules are required before deployment claims are warranted.

\paragraph{External validity and repair portability.}
The action set $\mathcal{A}_{\mathrm{real}}$ is calibrated from a limited set of models, prompts, tasks, and tool interfaces.  Its bounds are not universal.  Model updates or runtime changes may invalidate the estimated elasticities and require recalibration.  Evaluation across independent runtimes and model families remains unfinished.

\paragraph{Private-data reproducibility and runtime diagnostics.}
Some learning and replay data cannot be released at record level.  The paper therefore omits organization-specific platform identifiers, collection procedures, and workflow metadata.  This protects confidentiality but limits independent reproducibility.  In addition, environment-dependent fallback events have not been fully attributed; health checks, versioned adapters, and structured failure-reason logging are unfinished.

\section{Conclusion}

Resource control in multi-turn LLM agents can be usefully framed as a sequential leader-follower problem rather than static rule tuning.  For the restricted formulation, we establish equilibrium existence and conditional bounds for response stability, projection loss, and surrogate-to-real regret transfer.  The framework learns responses over context, prompt verbosity, and tool budget, then repairs the resulting controller using real-API evidence.  In the primary experiment, the selected repaired controller reduces mean token use relative to a conservative baseline, while the measured quality difference is not statistically significant.  Because the regret and transfer constants are not estimated, exact equilibrium, online effectiveness, portability, and production reliability remain open.

\section*{Impact Statement}

This work aims to improve the cost, reliability, and controllability of LLM agent systems.  Lowering unnecessary resource consumption may reduce operational cost and energy usage, but resource-control policies can also degrade user experience if quality signals are misspecified or if tools are under-allocated for complex tasks.  For this reason, we advocate strict passive shadow validation, conservative action-space repair, and explicit safety gates before active deployment.

\bibliography{stackelberg_resource_aware_agents_report_revised}
\bibliographystyle{plainnat}

% \appendix
% \onecolumn

% \section{Action Interpretation}

% The three action dimensions are interpreted as resource-control intensities.  Low context values correspond to aggressive compression, intermediate values to smart summaries, and high values to broad context retention.  Prompt values correspond to concise, standard, or detailed prompting.  Tool values correspond to no tools, restricted core tools, or broader tool budget.  The repaired deployment caps tool budget at 0.50.

% \section{Shadow Log Schema}

% A shadow record contains coarse timing information, coarse task type, a resource-action vector, fallback category, optional aggregate token statistics, and a quality proxy.  Privacy-preserving exports aggregate these records and omit raw prompts, assistant messages, free-form notes, direct identifiers, and per-turn records.

% \section{Coding-Aware Repair Details}

% The coding-aware repair head is trained from privacy-filtered task traces with labels designed to estimate minimum necessary tool budget.  It is intentionally not trained to imitate historical tool usage, because historical behavior may overuse tools.  In passive mode, the repair only modifies the logged counterfactual recommendation.

\end{document}